\documentclass[a4paper]{article}
\usepackage[margin=1in]{geometry}
\usepackage{lmodern}
\usepackage{amsmath,amssymb,amsthm}
\usepackage{mathtools}
\usepackage{enumitem}
\usepackage[dvipsnames]{xcolor}
\usepackage[authoryear,round]{natbib}
\usepackage[colorlinks=true,linkcolor=red!45!black,citecolor=blue!55!black,urlcolor=blue!70!black,hypertexnames=false]{hyperref}

\newtheorem{theorem}{Theorem}[section]
\newtheorem{lemma}[theorem]{Lemma}

\theoremstyle{remark}
\newtheorem{note}{Note}
\usepackage{thm-restate}

\usepackage{dsfont}

\DeclareMathOperator{\bbE}{\mathbb E}
\DeclareMathOperator{\bbP}{\mathbb P}
\newcommand{\bbR}{\mathbb R}

\newcommand{\cT}{\mathcal T}

\newcommand{\cN}{{\mathcal N}}
\newcommand{\cE}{\mathcal{E}}
\newcommand{\cV}{\mathcal{V}}

\newcommand{\cX}{{\mathcal X}}

\newcommand{\Err}{\operatorname{Err}}

\newcommand{\ba}{\boldsymbol a}

\newcommand{\bZ}{\boldsymbol Z}
\newcommand{\bw}{\boldsymbol w}
\newcommand{\be}{\boldsymbol e}
\newcommand{\bb}{\boldsymbol b}
\newcommand{\frakc}{\mathfrak c}

\title{Optimal Lower Bounds for Networked Information Aggregation}
\author{
    Ambar Pal\footnote{This work is not related to AP's position at Amazon.}\\
    {\small ambarpal@amazon.com}\\
    Amazon Responsible AI
}
\date{\today}

\begin{document}
\maketitle

\begin{abstract} 
The problem of networked information aggregation, studied in \citet{krr}, involves a group of learners situated on the vertices of a directed acyclic graph $G$, each learning a linear predictor $\widehat Y$ for a fixed random variable $Y$ given access to a local feature, as well as the predictors learnt by its parents. Learning proceeds iteratively, with learners ordered according to a topological sort of $G$. The main quantity of interest is the error incurred by the current learner, constrained to this flow of information, with respect to the best linear predictor using all the features seen so far. When the studied error is the MSE, i.e., $\bbE (\widehat Y - Y)^2$, \citet{krr} show that the error is at most $O(1/\sqrt{D})$ along a path of length $D$. They also obtain a hard instance where the {\rm MSE} is lower bounded by $\Omega(1/D)$, leaving the correct order open. In this work, we resolve this central open problem, and obtain a family of worst case problem instances with a {\rm MSE} lower bound of $\Omega(1/\sqrt{D})$. %

By exploiting invariances in the structure of the learnt predictors, our analysis generalizes to all loss functions $\ell(\widehat Y, Y)$ satisfying regularity conditions which include strong convexity in a ball around the origin, and that the ideal predictor minimizing the population loss is positively correlated with the label. We show that networked information aggregation on a gaussian instance in our worst case family incurs an $\ell$-error lower bounded by $\Omega(1/\sqrt{D})$ with respect to this ideal predictor. We demonstrate that a variety of common losses satisfy these regularity conditions. In particular, the logistic loss satisfies them, and hence our analysis also closes the gap between the upper and lower bounds in \citet{bateninetworked}.
\end{abstract}

\tableofcontents

\section{Problem Setup} \label{sec:setup}
We have a group of learners $\cV$ each of which try to learn a predictor $\widehat Y_i$ for a random variable $Y$, that obtains a low population loss $\bbE \ell(\widehat Y_i, Y)$. %
These learners are arranged in the form of a directed acyclic graph $G = (\cV, \cE)$, where an edge $(i, j) \in \cE$ exists whenever learner $j$ has access to the predictor learnt by learner $i$. %
Each learner has knowledge of a local set of \emph{features}, i.e., random variables $\cX_i$ that may be predictive of $Y$. Additionally, each learner $i$ has access to the predictors $\widehat Y_j$ learnt by its parents in $G$. %

Learning proceeds iteratively following along vertices ordered by a topological sort of $G$. At each iteration, the current learner $j$ obtains a linear predictor $\widehat Y_j$ by solving the optimization problem $\min_{\widehat Y} \bbE \ \ell(\widehat Y, Y)$ over $\widehat Y$ constrained to be linear in the features accessible to learner $j$, i.e., $\cX_j \cup \{\widehat Y_i \colon (i, j) \in \cE\}$. For the purposes of our analysis, it will be sufficient to assume that both these sets are singletons, i.e., $\cX_j = \{X_j\}$ for all learners $j$, and each $j$ has at most one parent $\text{par}(j)$ in $G$.

For $\ell(\widehat Y, Y) = (\widehat Y - Y)^2$, the above becomes the so-called \emph{networked information aggregation} setup studied by \citet{krr}. In this work, we will analyze a generalization of the worst case instance for this problem studied by \citet{krr}, and show that the error decays along the path at a rate inversely proportional to the square root of the length of the path. Further, we will see that the same lower bound applies to logistic regression, studied in \citet{bateninetworked}. 

\section{Main Results}
\paragraph{Worst Case Instance Family.} Let $Z_1, \ldots, Z_k$ be random variables such that $\bbE Z_i = 0, \bbE Z_i^2 = 1$ for $i \in \{1, 2, \ldots, k\}$, and $\bbE Z_i Z_j = 0$ whenever $i \neq j$ for $i, j \in \{1, 2, \ldots, k\}$. Let the learners be $\cV=\{1, 2, \ldots, k \cdot (k - 1)\}$, and the edges $\cE$ be given by
\[
1 \longrightarrow 2 \cdots \longrightarrow k
\longrightarrow k + 1 \longrightarrow k + 2 \longrightarrow \cdots \longrightarrow 2k \longrightarrow \cdots \longrightarrow k \cdot (k - 1).
\]
We define the \emph{pass} $p$ to be the group of learners 
\[
\{(p - 1) \cdot k  + 1, (p - 1) \cdot k + 2, \ldots, (p - 1) \cdot k + k\}.
\]
For each vertex $i$, we let the local feature set be a singleton, given by the random variable $X_i$. For the first vertex on pass $p$, we let $X_{(p - 1) \cdot k + 1} = Z_1$, and for each $2 \leq i \leq k$, we let $X_{(p - 1) \cdot k + i} = Z_i - Z_{i-1}$. Note that the set of all features viewed in each pass is the same. Finally, we define the label $Y = Z_k$. 

Recall that each learner solves an optimization program to minimize the population loss $L$ with respect to the features it has access to. We define $\Err^\ell(i)=\bbE \ell(\widehat Y_i, Y)$, and our results provide bounds on the error at the end of each pass $p$, defined as $\Err^\ell_p = \bbE \ell(\widehat Y_{p \cdot k}, Y)$ for $p \in \{1, 2, \ldots, k - 1\}$. Our results can be divided according to the loss used. 

\subsection{Least Squares Regression}
Setting the loss to be the mean squared error, i.e., $\ell_{\rm MSE}(\widehat Y, Y) = (\widehat Y - Y)^2$ recovers the setup studied in \citet{krr}. Each learner $i$ solves the optimization problem 
\begin{equation}\label{eq:vertex-program-mse}
\min_{\alpha,\beta\in\bbR}
\bbE \left[\bigl(\alpha X_i + \beta \widehat Y_{i - 1} - Y \bigr)^2\right]
\end{equation}
to obtain the predictor $\widehat Y_i = \alpha_i X_i + \beta_i \widehat Y_{i-1}$, with the convention $\widehat Y_0 = 0$. Note that the ideal linear predictor at the end of each pass, given access to all the features seen so far, obtains zero error, as $\widehat Y^\star = X_1 + X_2 + \ldots + X_k = Y$ obtains $\ell_{\rm MSE}(\widehat Y^\star, Y) = 0$. Our main result is
\begin{restatable}[Simplified MSE lower bound]{theorem}{mainMSELowerBound}\label{thm:main-mse}
For every instance in the worst case family, and every $1\leq p\leq k-1$,
\[
\Err^{\rm MSE}_p\ge\frac{1}{36\sqrt p}.
\]
\end{restatable}
For a fixed $k$, along with the matching upper bound of $O(\frac{1}{\sqrt{p}})$ in \citet[Corollary 3.10]{krr}, this settles the correct order of dependence on depth for learning with the MSE loss in the networked information aggregation model. Theorem~\ref{thm:main-mse} contains simplified constants for clarity, the full result is in Theorem~\ref{thm:regression-improved-constants}.

\subsection{Logistic Regression} \label{sec:logistic}
Setting $\ell$ to be the logistic loss, $\ell_{\rm logistic}(\widehat Y, Y) = -\sigma(Y) \log \sigma(\widehat Y) - (1 - \sigma(Y)) \cdot \log (1 - \sigma (\widehat Y))$, where $\sigma(y) = \frac{1}{1 + \exp(-y)}$ recovers the setup studied in \citet{bateninetworked}. Note that we require a slight transformation of the random variables in \citet{bateninetworked} to keep our exposition consistent: our $Y$ is a real-valued random variable, which is transformed into their binary label probability using $\sigma(Y)$, and similarly our real-valued $\widehat Y$ is viewed as the \emph{logit} learnt by each learner in their setup to produce the probability $\sigma(\widehat Y)$ of predicting the label $1$ in their setup. Similar to \eqref{eq:vertex-program-mse}, each learner $i$ now solves the optimization problem
\begin{equation*}
\min_{\alpha,\beta\in\bbR}
\bbE \left[ \ell_{\rm logistic} \bigl(\alpha X_i + \beta \widehat Y_{i - 1}, Y \bigr)\right].
\end{equation*}
Note that the ideal linear predictor at the end of each pass is still $\widehat Y^\star = \sum_i X_i = Y$, but it no longer achieves zero error, as $\ell_{\rm logistic}(Y, Y) = \ell_{\rm logistic}(Z_k, Z_k)$ is non-negative in general. We hence show 
\begin{restatable}[Simplified Logistic error lower bound]{theorem}{mainLogisticLowerBound}\label{thm:main-logistic}
For the instance in the worst case family where $Z_1, Z_2, \ldots, Z_k \overset{\rm i.i.d.}{\sim} \cN(0, 1)$, for every $1\leq p\leq k-1$,
\begin{equation*}
\Err^{\rm logistic}_p \ge\frac{1}{10^4 \sqrt p} + \Err^{\rm logistic, \star},
\end{equation*}
where $\Err^{\rm logistic, \star} = \bbE \ell_{\rm logistic}(Z_k, Z_k)$. 
\end{restatable}
For a fixed $k$, along with the matching upper bound of $O(\frac{1}{\sqrt{p}})$ in \citet[Theorem 3.8]{bateninetworked}, this settles the correct order of dependence on depth for learning with the logistic loss in the networked information aggregation model. Theorem~\ref{thm:main-logistic} contains simplified constants, and is in fact proved in Section~\ref{subsec:logistic} as an application of our more general result Theorem~\ref{thm:main-general}.

\subsection{General Convex Loss}
Finally, we generalize our lower bound arguments to a differentiable loss $\ell$. Subject to a few regularity conditions including strong convexity in a ball around the origin and differentiability, we show that,
\begin{theorem}[Simplified general error lower bound]\label{thm:main-general}
For the gaussian instance in the worst case family where $Z_1, Z_2, \ldots, Z_k \overset{\rm i.i.d.}{\sim} \cN(0, 1)$, for any convex loss $\ell$ satisfying \eqref{regular1} and \eqref{regular2}, for every $1 \leq p \leq k - 1$, we have,
\begin{equation*}
\Err^\ell_p\ge\frac{\mu}{72 \sqrt p} + \Err^{\ell, \star},
\end{equation*}
where $\Err^{\ell, \star} = \bbE \ell(Z_k, Z_k)$, and $\mu$ is the strong convexity parameter in \eqref{regular1}. 
\end{theorem}
\section{Overall proof technique}\label{sec:outline}
We now describe the high level proof technique, and then derive each result in the subsequent sections. Note that the learnt predictor is a linear combination of the latent variables, $Z_1, Z_2, \ldots, Z_k$, which we will denote by $\sum_i a_i Z_i$ in the following, for coefficients $\ba = (a_1, \ldots, a_k) \in \bbR^k$. Further, we will denote the coefficients at the last vertex of pass $p$ as $\ba^{(p)}$. The first four steps apply to learning with $\ell_{\rm MSE}$, and the rest extend the arguments to other losses.
\begin{enumerate}[leftmargin=*,itemsep=0.55em]
\item In each pass $p$, we divide the vertices into three types: Type I is the starting vertex $k \cdot (p - 1) + 1$, Type II consists of the intermediate vertices, $k \cdot (p - 1) + i$ for $2 \leq i \leq k-1$, and Type III is the final vertex $k \cdot p$. We then compute a closed form expression for the coefficients at each vertex as a transformation of the coefficients of its predecessor, in Lemma~\ref{lem:vertex-maps}. Specifically, each vertex type corresponds to one such transformation $T$: $T_1$ sets $a_1$ to zero and rescales, $T_i$, for $2 \leq i \leq k - 1$ averages two adjacent coefficients in $\ba$ and rescales, and $T_k$ averages the last two coefficients in $\ba$, rescales, and adds $(0,\ldots,0,-\tfrac12,\tfrac12)$.

\item Lemma~\ref{lem:invariants} then derives the coefficient invariants at the end of every pass. In particular,
\begin{equation*}
 \sum_j a_j=0,\qquad
 a_k=1-\Err^{\rm MSE}_p,\qquad
 a_k-a_{k-1}=1,\qquad
 a_{k-1}=-\Err^{\rm MSE}_p.
\end{equation*}
It also identifies that only the last $p + 1$ coefficients are non-zero (in $\ba^{(p)}$) at the end of pass $p$, and that all of these are negative except the last coefficient $a^{(p)}_k$. 

\item The zero-sum invariant, along with the fact that $\Err^{\rm MSE}_p = \sum_{i = k - p}^{k-1} a_i^2 + (1 - a_k)^2$ at the last vertex of pass $p$, is enough to recover the lower bound in \cite{krr}. To see this, apply Cauchy--Schwarz to
\begin{equation*}
a_{k-p}+\cdots+a_{k-1}+(a_k-1)=-1,
\end{equation*}
to obtain $\Err^{\rm MSE}_p \geq \frac{1}{p + 1}$. 
Lemma~\ref{lem:cumulative} generalizes the above idea, applying Cauchy--Schwarz to the last $N+1$ coefficients instead,
\begin{equation*}
a_{k-N}+\cdots+a_{k-1}+(a_k-1)=c_{k-N}-1,
\qquad \text{ where }
c_i=\sum_{j=i}^k a_j,
\end{equation*}
and obtains
\begin{equation*}
\Err^{\rm MSE}_p\ge\frac{(1-c_{k-N})^2}{N+1}.
\end{equation*}
Now, it is enough to show that $c_{k-N}$ is low enough (say $1 - c_{k-N} \geq 1/3$), for some $N \lesssim \sqrt p$. 

\item To control $c_{k-N}$, Lemma~\ref{lem:pass-recurrence} obtains the recurrence relationship relating $\ba^{(p)}$ to all the previous $\ba^{(i)}, i < p$. This recurrence involves complicated non-linear terms in the coefficients through the normalization factors. However, a technical manipulation allows us to bypass analysis of the normalization terms by bounding the ratio of the partial sum of these coefficients to $(1 - \Err^{\rm MSE}_p)$. Lemma~\ref{lem:pass-recurrence} contains this argument, allowing us to obtain
\begin{equation*}
c_{k - N} \lesssim \exp \left(-\frac{N^2}{p}\right), 
\end{equation*}
completing our argument.
\end{enumerate}

The above argument can be neatly extended to more general losses satisfying some regularity conditions, including strong convexity, by utilizing a conditioning argument for regression under misspecified losses developed in \citet{li-duan1989}. This extension is presented in Section~\ref{sec:general-convex}. 

\section{Least Squares Regression}\label{sec:regression}
In this section, we analyze learning on \emph{any} problem instance in our worst case family, with the MSE loss $\ell_{\rm MSE} = (\widehat Y - Y)^2$, and provide details for steps (1-4) outlined in Section~\ref{sec:outline}. 
\subsection{Vertex Transformations}\label{sec:maps}
Every prediction is a linear combination of the latent variables $Z_1, \ldots, Z_k$. In the following, suppose that the prediction output by the parent of a vertex is
\begin{equation*}
\widehat Y=\sum_{j=1}^k a_j Z_j,
\qquad
\ba=(a_1,\ldots,a_k).
\end{equation*}
Then, Lemma~\ref{lem:vertex-maps} provides the prediction output of each vertex in terms of $\ba$. 
\begin{restatable}[Vertex Transformations]{lemma}{vertexMapsRestated}\label{lem:vertex-maps}
Define the linear operators $A_1, A_2, \ldots, A_k$, where $A_1\ba=(0,a_2,\ldots,a_k)$ set the first coefficient to zero, and $A_i\ba = \left(a_1, \ldots, \frac{1}{2}(a_{i - 1} + a_i), \frac{1}{2}(a_{i - 1} + a_i), a_{i+1}, \ldots, a_k\right)$ for $2 \leq i \leq k$ replace $a_{i-1}$ and $a_i$ by their average and leave the other coordinates unchanged. With the convention that $\frac{0}{0} = 0$, the transformation at the $i$-th vertex in pass $p$ is as follows.
\begin{enumerate}[label=(\roman*),leftmargin=*,itemsep=0.35em]
\item At the Type-I learner $k(p-1)+1$, assigned $X_{k(p-1)+1}=Z_1$,
\begin{equation*}
T_1(\ba)=s_1(\ba)A_1\ba,
\qquad
s_1(\ba)=\frac{a_k}{\sum_{j=2}^k a_j^2}.
\end{equation*}
\item At a Type-II learner $k(p-1)+i$, assigned $X_{k(p-1)+i}=Z_i-Z_{i-1}$ for $2\leq i\leq k-1$,
\begin{equation*}
T_i(\ba)=s_i(\ba)A_i\ba,
\qquad
s_i(\ba)=\frac{a_k}{\lVert A_i\ba\rVert_2^2}.
\end{equation*}
\item At the Type-III learner $kp$, assigned $X_{kp}=Z_k-Z_{k-1}$,
\begin{equation*}
T_k(\ba)=s_k(\ba)A_k\ba+\bb,
\qquad
s_k(\ba)=
\frac{a_{k-1}+a_k}
{2\sum_{j=1}^{k-2}a_j^2+(a_{k-1}+a_k)^2}
\qquad
\bb=(0,\ldots,0,-\tfrac12,\tfrac12).
\end{equation*}
\end{enumerate}
\end{restatable}

\begin{note}\label{note:scale-invariance}
Observe that each vertex map $T_i$ is invariant to positive scaling, i.e., for every $t>0$, $T_i(t\ba)=T_i(\ba)$. Consequently, when we compose the maps in one pass, we ignore all intermediate positive normalizations until the final Type-III update.
\end{note}

\subsection{Coefficient invariants}\label{sec:invariants}

Recall from Section~\ref{sec:outline} that
\begin{equation*}
\ba^{(p)}=(a_1^{(p)},\ldots,a_k^{(p)})
\end{equation*}
is defined to be the coefficient vector at the end of pass $p$. Lemma~\ref{lem:invariants} then obtains a few important properties of coefficient vectors that the learners output. 
\begin{restatable}[Coefficient invariants]{lemma}{coefficientInvariantsRestated}\label{lem:invariants}
For the coefficient vector $\ba$ returned by any vertex, the last coordinate is related to the {\rm MSE} error as
\begin{equation}\label{eq:ak-one-minus-err}
a_k=1-\Err^{\rm MSE},
\end{equation} %
the coefficients sum to zero,
\begin{equation}\label{eq:coefficient-zero-sum}
\sum_{j=1}^k a_j=0,
\end{equation}
and the squared norm is given by
\begin{equation}\label{eq:squared-sum}
\sum_{j=1}^k a_j^2=a_k.
\end{equation}

For the coefficient vector $\ba^i$ returned by vertex $k \leq i \leq k(k-1)$, we have that $\frac{1}{2} \leq a^i_k < 1$. Further, for $1 \leq p \leq k - 1$, for the coefficient vector $\ba^{(p)} = \ba^{p k}$ returned by the last vertex of pass $p$, the last two coordinates are related as
\begin{equation}\label{eq:last-coordinate-gap}
a_k^{(p)}-a_{k-1}^{(p)}=1,
\end{equation}
hence
\begin{equation}\label{eq:penultimate-coordinate}
a_{k-1}^{(p)}=-\Err^{\rm MSE}_p.
\end{equation}

Finally, the signs of the coefficient vector after each pass are given by
\begin{equation}\label{eq:coefficient-signs}
a_j^{(p)}=0\quad(j<k-p),
\qquad
a_j^{(p)}<0\quad(k-p\leq j<k),
\qquad
a_k^{(p)}\ge\frac12.
\end{equation}
\end{restatable}

\subsection{Cumulative-sum lower bound}\label{sec:generic}

Fix a pass $1 \leq p \leq k - 1$ and define
\begin{equation*}
c_i^{(p)}=\sum_{j=i}^k a_j^{(p)},\qquad 1\leq i\leq k.
\end{equation*}

\begin{restatable}[Cumulative-sum bound]{lemma}{cumulativeBoundRestated}\label{lem:cumulative}
For every integer $0\leq N\leq k-1$,
\begin{equation}\label{eq:cumulative-bound}
\Err^{\rm MSE}_p\ge\frac{(1-c_{k-N}^{(p)})^2}{N+1}.
\end{equation}
\end{restatable}

Taking $N=p$, equations \eqref{eq:coefficient-zero-sum} and \eqref{eq:coefficient-signs} give $c_{k-p}^{(p)}=0$, and therefore
\begin{equation}\label{eq:krr-bound}
\Err^{\rm MSE}_p\ge\frac1{p+1},
\end{equation}
showing the error lower bound of \citet[Theorem~5.3]{krr} is valid for any instance in our worst case family. Our improvement over \eqref{eq:krr-bound} comes from showing that $c_{k-N}^{(p)}$ is small even for $N \lesssim \sqrt p$.

\subsection{Controlling the cumulative sum}\label{sec:control}

In this section, we upper bound $c^{(p)}_{k-N}$. In principle, one could directly try to iterate the recurrence $a^{(p)} = T_k \circ T_{k-1} \circ \ldots \circ T_1 (a^{(p-1)})$ to obtain $a^{(p)}$, and calculate the cumulative sum from the resultant vector. This direct approach does not seem to work, as keeping track of the normalization factors accumulating at each pass is difficult. In the following, we instead tightly utilize the invariances we have developed in Lemma~\ref{lem:invariants} to circumvent tracking the normalization factors, and solve a more manageable recurrence.

For ease of exposition, we change coordinates to count from the right edge of the vector $\ba^{(p)}$. We ignore the only positive entry $a_k^{(p)}$, since it is determined by the remaining entries through the zero-sum invariant. Define
\begin{equation}\label{eq:edge-profile}
w_i^{(p)}=-a_{k-i}^{(p)},\qquad 1\leq i\leq k-1.
\end{equation}
From the coefficient invariances \eqref{eq:ak-one-minus-err}, \eqref{eq:coefficient-zero-sum}, \eqref{eq:penultimate-coordinate}, and \eqref{eq:coefficient-signs}, we get $w_i^{(p)}\ge0$ for all $1\leq i\leq k-1$,
\begin{equation}\label{eq:edge-facts}
w_1^{(p)}=\Err^{\rm MSE}_p, \qquad
\sum_{i=1}^{k-1}w_i^{(p)}=c_k^{(p)} = 1-\Err^{\rm MSE}_p, \qquad
\sum_{i=N+1}^{k-1}w_i^{(p)} = c^{(p)}_{k-N}.
\end{equation}

\paragraph{Proof Strategy.} We pause here to comment on the proof strategy, as it involves an important manipulation. We will show that $c^{(p)}_{k - N}/{c^{(p)}_k} \leq \gamma_p(N)$ for some small $\gamma_p(N)$, which will then imply via \eqref{eq:edge-facts} that $c^{(p)}_{k - N} \leq \gamma_p(N) (1 - \Err_p^{\rm MSE}) < \gamma_p (N)$. This manipulation incurs a slack of $\Err_p^{\rm MSE}$ due to the second inequality, but, as we will see, allows us to go around reasoning about the scaling altogether, due to the ratio $c^{(p)}_{k - N}/{c^{(p)}_k}$.
\begin{restatable}[One-pass recurrence]{lemma}{onePassRecurrenceRestated}\label{lem:pass-recurrence}
There are positive normalization factors $\eta_p$, and a linear transformation $\cT: \bbR^{k-1} \to \bbR^{k-1}$ such that for $1 \leq p \leq k - 2$, 
\begin{equation}\label{eq:vector-recurrence}
\bw^{(p+1)}
=\eta_p\cT\bw^{(p)}+\Err^{\rm MSE}_{p+1}\be_1,
\qquad
\bw^{(1)}=\tfrac12\be_1,
\end{equation}
where
\begin{equation}\label{eq:T-def}
(\cT w)_i=
\begin{cases}
0, & i=1,\\
\displaystyle\sum_{t=0}^{k-i}2^{-(t+1)}w_{i-1+t},
& 2\leq i\leq k-1.
\end{cases}.
\end{equation}

Consequently (letting an empty product equal $1$),
\begin{equation}\label{eq:positive-mixture}
\bw^{(p + 1)}
=\sum_{s=1}^{p + 1}
\Err^{\rm MSE}_s\left(\prod_{r=s}^{p}\eta_r\right)
\cT^{p + 1 - s}\be_1.
\end{equation}
\end{restatable}

Lemma~\ref{lem:pass-recurrence} thus shows that the vector $\bw^{(p+1)}$ is a weighted sum of the $p+1$ vectors $\cT^n\be_1$ for $0\leq n\leq p$. This is useful for our purposes, as the cumulative sums of $\bw^{(p+1)}$ is then equal to a weighted sum of the cumulative sums of $\cT^n\be_1$. The following lemma lets us use this structure, by calculating the cumulative sums of $\cT^n\be_1$.

\begin{restatable}[Cumulative sums]{lemma}{partialSumsRestated}\label{lem:spike-tail}
For $0\leq p\leq k-2$ and $0\leq N\leq k-1$, define the cumulative sum
\begin{equation}\label{eq:partial-sum-definition}
\frakc(p,N)
=\sum_{i=N+1}^{k-1}\left(\cT^p\be_1\right)_i.
\end{equation}
Then, we have
\begin{equation}\label{eq:spike-tail}
\frakc(p,N)=
\begin{cases}
\displaystyle
\binom{2p-N}{p}2^{-(2p-N)},&0\leq N\leq p,\\[0.8em]
0,&p<N\leq k-1.
\end{cases}
\end{equation}
Consequently, for $1\leq N\leq p\leq k-2$,
\begin{equation}\label{eq:tail-ratio}
\frac{\frakc(p,N)}{\frakc(p,0)}
=\prod_{j=0}^{N-1}\frac{2(p-j)}{2p-j}
\leq\exp \left(-\frac{N(N-1)}{4p}\right).
\end{equation}
\end{restatable}

The proof of Lemma~\ref{lem:spike-tail} is done cleanly using the theory of generating functions, and is deferred to Appendix~\ref{app:partial-sums-proof}. We now have the machinery needed to prove our main Theorem~\ref{thm:main-mse}.
\begin{proof}[Proof of Theorem~\ref{thm:main-mse}]
Note that for $p \leq 3$, we have $\frac{1}{p + 1} \geq \frac{1}{36 \sqrt{p}}$, so \eqref{eq:krr-bound} is sufficient to show the result. For $p \geq 4$, we have
\begin{align}
c_{k-N}^{(p)}
&\overset{\eqref{eq:edge-facts}}{=}\sum_{i=N+1}^{k-1}w_i^{(p)}\notag\\
&\overset{\eqref{eq:positive-mixture}}{=}\sum_{i=N+1}^{k-1}\sum_{s=1}^p
\Err^{\rm MSE}_s\left(\prod_{r=s}^{p-1}\eta_r\right)
\left(\cT^{p-s}\be_1\right)_i\notag\\
&\overset{\eqref{eq:partial-sum-definition}}{=}\sum_{s=1}^p
\underbrace{\Err^{\rm MSE}_s\left(\prod_{r=s}^{p-1}\eta_r\right)}_{>0}
\frakc(p-s,N)\notag\\
&
\overset{\eqref{eq:tail-ratio}}{\leq}
\exp \left(-\frac{N(N-1)}{4(p-1)}\right)
\sum_{s=1}^p
\Err^{\rm MSE}_s\left(\prod_{r=s}^{p-1}\eta_r\right)
\frakc(p-s,0) \label{eq:pszero}\\
&=
\exp \left(-\frac{N(N-1)}{4(p-1)}\right)
\sum_{i = 1}^{k - 1}\left(
\sum_{s=1}^p
\Err^{\rm MSE}_s\left(\prod_{r=s}^{p-1}\eta_r\right)
\cT^{p - s} \be_1
\right)_i\notag\\
&\overset{\eqref{eq:edge-facts}}{=}
\exp \left(-\frac{N(N-1)}{4(p-1)}\right)
(1-\Err^{\rm MSE}_p)\notag\\
&\leq
\exp \left(-\frac{N(N-1)}{4(p-1)}\right),
\label{eq:c-tail-bound}
\end{align}
where \eqref{eq:pszero} follows from \eqref{eq:tail-ratio} when $1 \leq N \leq p - s \leq p - 1$, and $\frakc(p - s, N) = 0$ when $N > p - s$.
Let $N=\lceil2\sqrt p\rceil$. Then $N \geq 2 \sqrt{p} \geq 2$ and $N - 1 \geq 2 \sqrt{p} - 1 \geq \sqrt{p}$, implying $N (N - 1) \geq 2p \geq 2(p - 1)$. Equation \eqref{eq:c-tail-bound} now yields
\begin{equation*}
c_{k-N}^{(p)}\leq e^{-1/2}<\frac23.
\end{equation*}
Now applying Lemma~\ref{lem:cumulative}, we have
\begin{equation*}
\Err^{\rm MSE}_p
\ge\frac{(1-c_{k-N}^{(p)})^2}{N+1}
\ge\frac{(1/3)^2}{4\sqrt p}
=\frac1{36\sqrt p}.
\end{equation*}
Finally, for $p = 1$, we have $\Err_p^{\rm MSE} = 1/2$, which satisfies the above. 
\end{proof}

\paragraph{Improving constants. } The constant $1/36$ is loose in the above: Theorem~\ref{thm:regression-improved-constants} performs improved book-keeping to show a tighter lower bound.
\begin{restatable}{theorem}{regressionConstantsRestated}\label{thm:regression-improved-constants}
For every $2\leq p\leq k-1$ and $t>0$ such that $\lceil t\sqrt{p-1}+1\rceil\leq k-1$,
\begin{equation}\label{eq:constant-family-t}
\Err^{\rm MSE}_p \geq \frac{(1-e^{-t^2/4})^2}{t\sqrt{p-1}+3}.
\end{equation}
\end{restatable}

Using $t=3$ in Theorem~\ref{thm:regression-improved-constants}, gives approximately $\Err^{\rm MSE}_p \geq \frac{1}{3.75\sqrt p}$ when $p$ and $k$ are large enough.

\section{Regression with a Convex Loss} \label{sec:general-convex}
In this section, we analyze learning with a general loss $\ell(\widehat Y, Y)$ that is convex in the first argument, subject to a few regularity conditions, which are satisfied by several common losses. Unlike Section~\ref{sec:regression}, we will specialize to the gaussian instance in our worst case family, assuming $Z_1, \ldots, Z_k \overset{\text{i.i.d.}}{\sim} \cN(0,1)$. Recall that $X_{k(p - 1) + j} = Z_{j} - Z_{j - 1}$ for $1 \leq j \leq k, 1 \leq p \leq k - 1$, and the predictor output by learner $i$ is $\widehat Y_i = \alpha_i X_i + \beta_i \widehat Y_{i - 1}$, where $(\alpha_i, \beta_i)$ is the solution to
\begin{equation}\label{eq:vertex-program-general}
\min_{\alpha,\beta\in\bbR}
\bbE \ell \bigl(\alpha X_i + \beta \widehat Y_{i - 1}, Z_k \bigr),
\end{equation}
with $\widehat Y_0 = 0, Z_0 = 0$. Note that for all $i$, both $X_i$ and $\widehat Y_{i-1}$ are linear combinations of $\{Z_1, \ldots, Z_k\}$. Defining $L(\ba) = \bbE \ell (\sum_j a_j Z_j, Z_k)$, our first regularity condition requires that the population loss be differentiable, strongly convex on a radius $2$ neighborhood around the origin, and minimized at the true label:
\begin{equation}
L \text{ is differentiable, } \mu \text{-strongly convex on } \{\ba \colon \|\ba\|_2 \leq 2\}, \text{ and minimized at } \be_k \tag{\textsc{Regular-I}} \label{regular1}
\end{equation}

Further, the second regularity condition we impose on $\ell$ is that the predictor minimizing the population loss be positively correlated with the label, i.e., for a random variable $E \sim \cN(0, \gamma_0(1 - \gamma_0))$ independent of $Z_k$, and scalar $\frac12 \leq \gamma_0 \leq 1$, we have
\begin{equation}
\min_{\gamma \in \bbR} \bbE \ell(\gamma (E + \gamma_0 Z_k), Z_k) \text{ is attained at } \gamma \in (0, 2] \tag{\textsc{Regular-II}} \label{regular2}
\end{equation}

Using \eqref{regular1} and \eqref{regular2}, we give an inductive argument showing that learning with $\ell$ incurs a $\Omega(1/\sqrt{p})$ error relative to the best linear predictor, matching the dependence that we saw for learning with ${\rm MSE}$ in Section~\ref{sec:regression}. In particular, we will prove the following:
\begin{lemma} \label{lem:scale}
$\widehat Y_i$ is a \emph{positive} scaled multiple of $\widehat Y_{i, \rm MSE}$ for all $i \geq 0$, where $\widehat Y_{i, \rm MSE}$ and $\widehat Y_i$ are the predictors output by the learner $i$ when learning with ${\rm MSE}$ \eqref{eq:vertex-program-mse}, and $\ell$ \eqref{eq:vertex-program-general}, respectively.
\end{lemma}
\begin{proof}
The base case is verified trivially as $\widehat Y_0 = \widehat Y_{0, \rm MSE} = 0$. Assume that the inductive hypothesis is true for $i - 1$, i.e., $\widehat Y_{i-1} = \gamma_{i-1} \widehat Y_{i - 1, \rm MSE}$ for some $\gamma_{i-1} > 0$.  
Decompose $Z_k = \widehat Y_{i, \rm MSE} + E_1$, where $E_1$ is a random variable independent of $(X_i, \widehat Y_{i-1})$. It can be seen that $E_1$ is not zero.
The key observation needed now comes from \citet{li-duan1989}: their conditioning argument can be replicated by expanding \eqref{eq:vertex-program-general} as follows,
\begin{align*}
\bbE \ell \bigl(\alpha X_i + \beta \widehat Y_{i - 1}, Z_k \bigr) =
\bbE \bbE \left(\ell \bigl(\alpha X_i + \beta \widehat Y_{i - 1}, Z_k \bigr)  \vert \widehat Y_{i, \rm MSE}, E_1 \right) 
\geq \bbE \left(\ell \bigl( \bbE (\alpha X_i + \beta \widehat Y_{i - 1} \vert \widehat Y_{i, \rm MSE}, E_1 \bigr), Z_k) \right)
\end{align*}
where the inequality is due to Jensen's inequality applied to the convex function $\ell(\cdot, z)$ for any fixed $z$. Further, $E_1 \perp (X_i, \widehat Y_{i-1})$ implies $E_1 \perp \widehat Y_{i, \rm MSE}$, as $\widehat Y_{i, \rm MSE}$ is a linear combination of $X_i$ and $\widehat Y_{i-1, \rm MSE} = (1/ \gamma_{i-1}) \widehat Y_{i-1}$. Hence, 
\begin{align*}
\bbE (\alpha X_i + \beta \widehat Y_{i - 1} \vert\widehat Y_{i, \rm MSE}, E_1) = \bbE (\alpha X_i + \beta \widehat Y_{i - 1} \vert\widehat Y_{i, \rm MSE}).
\end{align*}
Now, $\alpha X_i + \beta \widehat Y_{i - 1}$ and $\widehat Y_{i, \rm MSE}$ are both linear combinations of $\{Z_1, \ldots, Z_k\}$, i.e., they are zero-mean, jointly gaussian random variables. Hence for some $\gamma \in \bbR$, we have
\begin{align*}
\bbE (\alpha X_i + \beta \widehat Y_{i - 1} \vert\widehat Y_{i, \rm MSE}) = \gamma \widehat Y_{i, \rm MSE}.
\end{align*}
We have thus shown that the objective of \eqref{eq:vertex-program-general}, at every feasible predictor is lower bounded by the objective at a predictor of the form $\gamma \widehat Y_{i, \rm MSE}$. Further, this predictor is feasible for \eqref{eq:vertex-program-general}, as $\widehat Y_{i-1} = \gamma_{i-1} \widehat Y_{i-1, \rm MSE}$ with $\gamma_{i-1} \neq 0$. This shows that there always exists a minimizer of \eqref{eq:vertex-program-general} of the form $\gamma \widehat Y_{i, \rm MSE}$, for some $\gamma \in \bbR$.  The optimization problem \eqref{eq:vertex-program-general} now reduces to 
\begin{equation}
\min_{\gamma \in \bbR} \bbE \ell(\gamma \widehat Y_{i, \rm MSE}, Z_k). \notag
\end{equation}
For completing the argument, we need to show that this minimizer $\gamma^\star$ exists, is positive, and unique, implying that it is indeed output by learner $i$. We handle $i \geq k$ and $1 \leq i \leq k - 1$ separately.  

For all $i \geq k$, by Lemma~\ref{lem:invariants}, we have $\widehat Y_{i, \rm MSE} = E_2 +  a^i_k Z_k$, for some $\frac12 \leq a^i_k \leq 1$, and $E_2$ a gaussian independent of $Z_k$ with ${\rm Var}(E_2) = a^i_k (1 - a^i_k)$. Then, \eqref{regular2} implies that $\gamma^\star \widehat Y_{i, \rm MSE}$ is a minimizer of \eqref{eq:vertex-program-general} for $\gamma^\star \in (0, 2]$. Further, the coefficients of $\widehat Y_{i, \rm MSE}$ satisfy $\|\ba^i\|_2^2 = a^i_k < 1$ (Lemma~\ref{lem:invariants}) implying that $\gamma^\star \widehat Y_{i, \rm MSE}$ lies in the region where \eqref{regular1} guarantees that the population loss is strongly convex. Along with the convexity of $\ell$, this ensures that $\gamma^\star \widehat Y_{i, \rm MSE}$ is the unique minimizer. 

For $1 \leq i \leq k - 1$, $\widehat Y_{i, \rm MSE} = 0$, implying that the zero predictor is a minimizer for \eqref{eq:vertex-program-general}. Now $\ell$ is convex, and \eqref{regular1} guarantees strong convexity near $\mathbf{0}$, implying that the zero predictor is the unique minimizer.   
\end{proof}

Let the predictor\footnote{Recall that we refer to the coefficients $\ba$ instead of the predictor $\widehat Y = \sum_j a_j Z_j$ when non-ambiguous from context.} at the end of pass $p$ be $\ba^{\rm MSE}$ when learning on the gaussian instance with $\ell_{\rm MSE}$, and $\ba^\ell$ when learning with a loss $\ell$. From Lemma~\ref{lem:scale}, we have that $\ba^\ell = \gamma \ba^{\rm MSE}$, and obtain
\begin{align}
\|\ba^{\rm \ell} - \be_k\|_2^2 
&= \| \gamma \ba^{\rm MSE} - \be_k\|_2^2 \notag \\
&= \gamma^2 \|\ba^{\rm MSE}\|_2^2 - 2 \gamma a^{\rm MSE}_k + 1 \notag \\
&= \gamma^2 a^{\rm MSE}_k - 2 \gamma a^{\rm MSE}_k + a^{\rm MSE}_k - a^{\rm MSE}_k + 1 \notag \\
&= (1 - \gamma)^2 a^{\rm MSE}_k + (1 - a^{\rm MSE}_k) \label{al-dist}  \\
&\geq \Err^{\rm MSE}_p, \notag 
\end{align}
where we have invoked Lemma~\ref{lem:invariants} on multiple steps, and define $\Err^{\rm MSE}_p$ as the error obtained by the predictor $\ba^{\rm MSE}$. Now using Theorem~\ref{thm:main-mse}, we have shown the following result.
\begin{theorem}[{\rm MSE} when learning with a regular loss] 
\label{thm:mse-misspecified}
For learning with a loss $\ell$ satisfying \eqref{regular1} and \eqref{regular2}, for the gaussian instance in the worst case family where $Z_1, Z_2, \ldots, Z_k \overset{\rm i.i.d.}{\sim} \cN(0, 1)$, for every $1 \leq p \leq k - 1$, we have,
\begin{equation*}
\bbE (\widehat Y_{(p)} - Y)^2 \geq \frac{1}{36 \sqrt{p}},
\end{equation*}
where $\widehat Y_{(p)}$ denotes the predictor learnt by the last vertex on pass $p$.
\end{theorem}
To obtain a bound on $\Err^\ell$ instead, we use \eqref{regular2} on \eqref{al-dist} to get $\|\ba^\ell - \be_k\|^2_2 \leq a^{\rm MSE}_k + (1 - a^{\rm MSE}_k) \leq 1$. Along with $\|\be_k\|_2 \leq 1$, this shows $\|\ba^\ell\|_2 \leq 2$. Thus, when $\ell$ satisfies both \eqref{regular2} and \eqref{regular1}, we have 
\begin{align*}
\Err^\ell_p - \Err^{\ell, \star}
= L(\ba^\ell) - L(\be_k) 
\geq (\ba^\ell - \be_k)^\top \nabla L(\be_k) + \frac{\mu}{2} \|\ba^\ell - \be_k\|^2 \geq \frac{\mu}{72 \sqrt{p}}.
\end{align*}
We have hence proved Theorem~\ref{thm:main-general}, and can apply it to several losses to obtain a $\Omega(\frac{1}{\sqrt{p}})$ error lower bound as corollaries. For the following discussion, fix a scalar $\gamma_0 \in [1/2, 1]$, and let $E \sim \cN(0, \gamma_0(1 - \gamma_0))$ be independent of $Z_k$. Further, denote the vector $\bZ = (Z_1, Z_2, \ldots, Z_k)$. 

\subsection{MSE} \label{subsec:mse}
For $L_{\rm MSE}(\ba) = \bbE (\sum_j a_j Z_j - Z_k)^2$, we have \[L_{\rm MSE}(\ba) = \bbE (\bZ^\top (\ba - \be_k))^2 = \bbE (\ba - \be_k)^\top \bZ \bZ^\top (\ba - \be_k) = \|\ba - \be_k\|_2^2,\] which is differentiable, strongly convex with $\mu = 2$, and minimized at $\be_k$. Thus \eqref{regular1} is verified. Further, \[\bbE (\gamma (E + \gamma_0 Z_k) - Z_k)^2 = \gamma^2 \gamma_0(1 - \gamma_0) + (\gamma \gamma_0 - 1)^2 = \gamma^2 \gamma_0 - 2\gamma \gamma_0 + 1 = (\gamma - 1)^2 \gamma_0 + (1 - \gamma_0),\] which has a unique minimizer at $\gamma = 1$. Thus, \eqref{regular2} is verified, and we recover the $\Omega(\frac{1}{36 \sqrt{p}})$ lower bound in Theorem~\ref{thm:main-mse}, as expected. 

\subsection{Logistic Loss} \label{subsec:logistic}
Recall that $L_{\rm logistic}(\ba) = \bbE \ell_{\rm logistic}(\ba^\top \bZ, Z_k)$, where $\ell_{\rm logistic}(\widehat y, y) = -\sigma(y) \log \sigma(\widehat y) - (1 - \sigma(y)) \log (1 - \sigma(\widehat y))$ with $\sigma(t) = (1 + \exp (-t))^{-1}$. Taking derivatives, we get
\[
\frac{\partial}{\partial \hat y} \ell_\text{logistic}(\hat y, y) = \sigma(\hat y) - \sigma(y),\text{ and}\quad \frac{\partial^2}{\partial \hat y^2} \ell_\text{logistic}(\hat y, y) =  \sigma(\hat y) (1 - \sigma(\hat y)) > 0 \text{ for finite } \widehat y. 
\]
Thus, $L$ is differentiable and convex everywhere, and strongly convex in any bounded domain, e.g., $\{ \ba \colon \| \ba \|_2 \leq 2 \}$. Further, $\nabla L(\ba) = \bbE \left( (\sigma(\ba^\top \bZ) - \sigma(Z_k)) \bZ \right) = 0$ when $\ba = \be_k$, implying that $L$ has a unique minimizer at $\be_k$. We still need to explicitly obtain the strong convexity parameter $\mu$. We have $\left[\nabla^2 L(\ba)\right]_{u, v} = \bbE \sigma'(\ba^\top \bZ) Z_u Z_v$, i.e., 
\[
\nabla^2 L(\ba) = \bbE \sigma'(\ba^\top \bZ) \bZ \bZ^\top.
\]
By rotational symmetry of the gaussian $\bZ$, we can replace $\bZ$ by $R \bZ$ for any rotation matrix $R$. Specifically, take $R$ such that $R^\top \ba = \|\ba\|_2 \be_1$. This gives $\nabla^2 L(\ba) = R \bbE \sigma'(Z_1 \|\ba\|_2) \bZ \bZ^\top R^\top$. As a result, we have
\begin{equation*}
\left[R^\top \nabla^2 L(\ba) R\right]_{u, v} = 
\begin{cases}
0, \quad&u \neq v \\
\bbE \sigma'(Z_1 \|\ba\|_2) Z_1^2, \quad&u = v = 1\\
\bbE \sigma'(Z_1 \|\ba\|_2), \quad&u = v \neq 1
\end{cases}.
\end{equation*}
Thus, $\nabla^2 L(\ba) \succcurlyeq \mu I_k$, where the strong convexity parameter $\mu = \inf_{\|\ba\|_2 \leq 2} \min\{\bbE \sigma'(Z_1 \|\ba\|_2) Z_1^2, \bbE \sigma'(Z_1 \|\ba\|_2)\}$. Lemma~\ref{lem:gaussian-lower-bound} completes this argument and gives $\mu \geq 0.048$. Thus, \eqref{regular1} is verified.

Define $F(\gamma) =  \bbE \ell_{\rm logistic}(\gamma (E + \gamma_0 Z_k), Z_k)$, and note that $F$ is strongly convex in $[0, 2]$. Further Lemma~\ref{lem:F-gradient}, shows that $F'(0) < 0$ and $F'(1) \geq 0$, showing that the unique minimizer $\gamma^\ast$ lies in $(0, 1]$, verifying \eqref{regular2}. Subsequently, applying Theorem~\ref{thm:main-general} proves Theorem~\ref{thm:main-logistic}.

\subsection{Huber and other Losses} \label{sec:other-losses}
For any fixed $\delta > 0$, define the Huber loss \citep{huber1964} as $L(\ba) = \bbE \ell_{{\rm huber}, \delta} (\ba^\top \bZ - Z_k) = \bbE \ell_{{\rm huber}, \delta} ((\ba - \be_k)^\top \bZ)$, where 
\begin{equation*}
\ell_{{\rm huber}, \delta}(t) = \begin{cases}
\frac{1}{2} t^2, &|t| \leq \delta, \\
\delta |t| - \frac{\delta^2}{2}, &|t| > \delta
\end{cases}.
\end{equation*}
As earlier, it is clear that $L$ is differentiable, strongly convex in $\{\ba \colon \|\ba\|_2 \leq 2\}$, and the minimizer is obtained at $\ba = \be_k$. We need to obtain the strong convexity parameter $\mu$. We have 
\[
\nabla L(\ba) = \bbE \left( \mathds{1}\left( |(\ba - \be_k)^\top \bZ| \leq \delta \right)  \bZ \bZ^\top (\ba - \be_k) + \mathds{1}\left( |(\ba - \be_k)^\top \bZ| > \delta \right) \delta {\rm sgn} (\bZ^\top (\ba - \be_k)) \bZ \right).
\]
Again, due to the rotational symmetry of the gaussian $\bZ$, we replace $\bZ$ by $R \bZ$ for a rotation matrix such that $R^\top (\ba - \be_k) = \|\ba - \be_k\|_2 \be_1$. This gives $\nabla^2 L(\ba) = \bbE \mathds{1}\left( \|\ba - \be_k\|_2 |Z_1| \leq \delta \right)  R \bZ \bZ^\top R^\top$. When $\ba = \be_k$, we have $\nabla^2 L(\ba) = I$. Otherwise, the expression simplifies to

\begin{equation*}
\left[ R^\top \nabla^2 L(\ba) R \right]_{u, v} = \begin{cases}
0, &u \neq v\\
\bbE \mathds{1} \left( |Z_1| \leq \frac{\delta}{\|\ba - \be_k\|_2} \right) Z_1^2, &u = v = 1\\
\bbP \left( |Z_1| \leq \frac{\delta}{\|\ba - \be_k\|_2} \right), &u = v \neq 1
\end{cases}.
\end{equation*}
As with $\|\ba - \be_k\|_2 \leq 3$, we get that the strong convexity parameter $\mu \geq \min \{ \bbE \mathds{1} ( |Z_1| \leq \delta/3 ) Z_1^2, \bbP (|Z_1| \leq \delta/3)\} > 0$. This completes verification of \eqref{regular1}. 

Further, defining $H(\gamma) =  \bbE \ell_{{\rm huber}, \delta}(\gamma (E + \gamma_0 Z_k) - Z_k)$, we note that similar to the {\rm MSE}, $\gamma (E + \gamma_0 Z_k) - Z_k \sim \cN(0, (\gamma - 1)^2 \gamma_0  + (1 - \gamma_0))$, hence, $H(\gamma) = \bbE \ell_{{\rm huber}, \delta} ( \sqrt{(\gamma - 1)^2 \gamma_0 + (1 - \gamma_0)} Z_1)$, which decreases as the variance of the random variable inside $\ell_{{\rm huber}, \delta}$ decreases. Hence, the unique minimizer is $\gamma = 1$, verifying \eqref{regular2}.

Very similar analyses can be performed for other losses, like $\ell(\widehat y, y) = \rho(\widehat y - y)$, or $\ell(\widehat y, y) =-y \widehat y + \rho(\widehat y)$ for a suitably normalized, strongly convex $\rho$. Finally, we note that the constant $2$ in the regularity conditions can be generalized to $c \geq 1$ to show the same result, with minor modifications in the proofs.

\section*{Acknowledgements and Disclosure of AI Usage}
We thank Aaron Roth and Michael Kearns for helpful discussions. Additionally, GPT 5.6 was used to discover the reference \citep{li-duan1989}, shorten lengthy algebra in the generating function proof of Lemma~\ref{lem:spike-tail}, discover the elementary proof of Lemma~\ref{lem:gaussian-lower-bound} by bounding the $\cosh$ function, and discover other losses that our argument applies to in Section~\ref{sec:other-losses}. 

\appendix
\section{Omitted Proofs for Section~\ref{sec:regression}}\label{app:omitted-proofs}

\subsection{Proof of Lemma~\ref{lem:vertex-maps}}
\vertexMapsRestated*

\begin{proof}
The objective in \eqref{eq:vertex-program-mse} is the squared Euclidean norm of the residual coefficient vector. We consider each vertex type separately.

At a Type-I learner, the objective is
\begin{align}
&\bbE\left[\left(\alpha Z_1+\beta\sum_{j=1}^k a_jZ_j-Z_k\right)^2\right]\nonumber\\
&\qquad=(\alpha+\beta a_1)^2
+\sum_{j=2}^{k-1}(\beta a_j)^2
+(\beta a_k-1)^2. \label{eq:type-I-objective}
\end{align}
For every fixed $\beta$, we may set $\alpha=-\beta a_1$, which makes the first term zero. The remaining objective is
\begin{equation*}
\beta^2\sum_{j=2}^k a_j^2-2\beta a_k+1.
\end{equation*}
It is minimized at
\begin{equation*}
\beta=\frac{a_k}{\sum_{j=2}^k a_j^2}=s_1(\ba).
\end{equation*}
The new coefficient vector is therefore
\begin{equation*}
\beta(0,a_2,\ldots,a_k)=s_1(\ba)A_1\ba,
\end{equation*}
which proves (i).

At a Type-II learner in position $i$, assigned $X_{k(p-1)+i}=Z_i-Z_{i-1}$, the objective is
\begin{align}
&\bbE\left[\left(\alpha(Z_i-Z_{i-1})
+\beta\sum_{j=1}^k a_jZ_j-Z_k\right)^2\right]\nonumber\\
&\qquad=
\sum_{j\notin\{i-1,i,k\}}(\beta a_j)^2
+(-\alpha+\beta a_{i-1})^2
+(\alpha+\beta a_i)^2
+(\beta a_k-1)^2. \label{eq:type-II-objective}
\end{align}
For fixed $\beta$, the terms containing $\alpha$ are minimized at
\begin{equation*}
\alpha=\frac{\beta}{2}(a_{i-1}-a_i).
\end{equation*}
At this value of $\alpha$, both affected predictor coefficients equal
\begin{equation*}
\frac{\beta}{2}(a_{i-1}+a_i).
\end{equation*}
Thus the new coefficient vector before optimizing over $\beta$ is $\beta A_i\ba$, and the objective becomes
\begin{equation*}
\beta^2\lVert A_i\ba\rVert_2^2-2\beta a_k+1.
\end{equation*}
It is minimized at
\begin{equation*}
\beta=\frac{a_k}{\lVert A_i\ba\rVert_2^2}=s_i(\ba),
\end{equation*}
which proves (ii).

At a Type-III learner, the objective is
\begin{align}
&\bbE\left[\left(\alpha(Z_k-Z_{k-1})
+\beta\sum_{j=1}^k a_jZ_j-Z_k\right)^2\right]\nonumber\\
&\qquad=
\sum_{j=1}^{k-2}(\beta a_j)^2
+(-\alpha+\beta a_{k-1})^2
+(\alpha+\beta a_k-1)^2. \label{eq:type-III-objective}
\end{align}
For fixed $\beta$, the terms containing $\alpha$ are minimized at
\begin{equation*}
\alpha=\frac{1+\beta(a_{k-1}-a_k)}{2}.
\end{equation*}
Substituting this value gives
\begin{equation*}
\beta^2\left(
\sum_{j=1}^{k-2}a_j^2
+\frac12(a_{k-1}+a_k)^2
\right)
-\beta(a_{k-1}+a_k)+\frac12.
\end{equation*}
The minimizing value of $\beta$ is
\begin{equation*}
\beta=
\frac{a_{k-1}+a_k}
{2\sum_{j=1}^{k-2}a_j^2+(a_{k-1}+a_k)^2}
=s_k(\ba).
\end{equation*}
At this value, the first $k-2$ predictor coefficients are $\beta a_j$, while the last two are
\begin{equation*}
\frac{\beta}{2}(a_{k-1}+a_k)-\frac12,
\qquad
\frac{\beta}{2}(a_{k-1}+a_k)+\frac12.
\end{equation*}
Hence the new coefficient vector is
\begin{equation*}
\beta A_k\ba+(0,\ldots,0,-\tfrac12,\tfrac12)
=s_k(\ba)A_k\ba+\bb,
\end{equation*}
which proves (iii).
\end{proof}

\subsection{Proof of Lemma~\ref{lem:invariants}}
\coefficientInvariantsRestated*

\begin{proof}
Fix a learner $i$ whose returned predictor has coefficient vector $\ba=(a_1,\ldots,a_k)$, and write
\begin{equation*}
Y=Z_k,
\qquad
\widehat Y=\widehat Y_i=\sum_{j=1}^k a_jZ_j,
\qquad
\Err=\Err(i)=\bbE[(Y-\widehat Y)^2].
\end{equation*}
By the least-squares program \eqref{eq:vertex-program-mse}, $\widehat Y$ is the orthogonal projection of $Y$ onto the feasible linear space
\begin{equation*}
\mathcal S_i
=\operatorname{span}\{X_i,\widehat Y_{i-1}\}.
\end{equation*}
Thus the residual $Y-\widehat Y$ is orthogonal to every element of $\mathcal S_i$. In particular, because $\widehat Y\in\mathcal S_i$,
\begin{equation}\label{eq:orthogonality}
\bbE[(Y-\widehat Y)\widehat Y]=0.
\end{equation}
As $\bbE[Z_r Z_s]=\mathbf 1\{r=s\}$, we have $\bbE[Y\widehat Y]=a_k$ and $\bbE[\widehat Y^2]=\sum_{j=1}^k a_j^2$. Substituting these two identities into \eqref{eq:orthogonality} gives $a_k=\sum_{j=1}^k a_j^2$, proving \eqref{eq:squared-sum}.

Finally, $\bbE[Y^2]=\bbE[Z_k^2]=1$, so expanding the squared error and using \eqref{eq:squared-sum} yields
\begin{equation*}
\Err
=\bbE[Y^2]-2\bbE[Y\widehat Y]+\bbE[\widehat Y^2]
=1-2a_k+\sum_{j=1}^k a_j^2
=1-a_k.
\end{equation*}
This proves \eqref{eq:ak-one-minus-err}. Note that this error is non-increasing along the path as every vertex may retain its incoming predictor by setting the coefficient of its newly assigned feature to zero. At the first Type-III learner, the error is $1/2$ by \eqref{eq:first-pass-vector}, hence every learner from that point onward has $0\leq\Err\leq1/2$. We will further obtain $\Err \neq 0$ shortly due to the zero-sum condition. Along with \eqref{eq:ak-one-minus-err}, this gives $1/2\leq a_k < 1$.

It remains to prove \eqref{eq:coefficient-zero-sum}, \eqref{eq:last-coordinate-gap}, \eqref{eq:penultimate-coordinate}, and \eqref{eq:coefficient-signs}. For $1\leq p\leq k-1$ and $1\leq i\leq k$, let $\ba^{(p,i)}$ be the vector returned by learner $k(p-1)+i$. Set $\ba^{(0)}=\boldsymbol 0$. Then
\begin{equation*}
\ba^{(p,0)}=\ba^{(p-1)},
\qquad
\ba^{(p,i)}=T_i(\ba^{(p,i-1)}),
\qquad
\ba^{(p,k)}=\ba^{(p)}.
\end{equation*}

The final update in every pass has the form
\begin{equation*}
\ba^{(p,k)}
=s_k(\ba^{(p,k-1)})A_k\ba^{(p,k-1)}+\bb.
\end{equation*}
The last two coordinates of $A_k\ba^{(p,k-1)}$ are equal, while those of $\bb$ are $-1/2$ and $1/2$. Their difference is therefore $1$, proving \eqref{eq:last-coordinate-gap}. Combining this identity with \eqref{eq:ak-one-minus-err} gives \eqref{eq:penultimate-coordinate}.

We first prove \eqref{eq:coefficient-zero-sum}. During the first pass, $T_i(\boldsymbol 0)=\boldsymbol 0$ for $i<k$, and the Type-III vertex outputs
\begin{equation}\label{eq:first-pass-vector}
\ba^{(1)}=(0,\ldots,0,-\tfrac12,\tfrac12),
\qquad
\Err_1=\tfrac12.
\end{equation}
Thus \eqref{eq:coefficient-zero-sum} holds throughout the first pass. At its end, only the last two coordinates can be nonzero. In each later pass, the left-to-right Type-II updates can extend the support by at most one coordinate to the left, while the Type-I and Type-III maps do not extend it further. Hence, before every Type-I update in the regime $p\leq k-1$, the first coordinate is zero, so the Type-I output is a scalar multiple of the same zero-sum vector. Every Type-II averaging operator preserves the coordinate sum, as does scalar multiplication. Finally, $A_k$ preserves the coordinate sum, and the entries of $\bb$ sum to zero. Induction over the passes proves \eqref{eq:coefficient-zero-sum} at every vertex.

We now prove \eqref{eq:coefficient-signs} by induction over the passes. Equation~\eqref{eq:first-pass-vector} is the base case. For the induction step, fix $2\leq p\leq k-1$, and assume that \eqref{eq:coefficient-signs} holds at the end of pass $p-1$. Since $a_1^{(p-1)}=0$, the Type-I map at the start of pass $p$ is the identity: by \eqref{eq:squared-sum}, its scale is $\frac{a_k^{(p-1)}}{\sum_{j=2}^k(a_j^{(p-1)})^2}=1$.

Hence $\ba^{(p,1)}=\ba^{(p-1)}$. Every Type-II scale is positive: its numerator satisfies $a_k>0$ by \eqref{eq:ak-one-minus-err} and error monotonicity, and its denominator is positive. Therefore, scaling does not change support or signs.

Before the Type-II updates, the induction hypothesis gives zero coordinates for $j\leq k-p$ and negative coordinates for $k-p+1\leq j\leq k-1$. The updates through $i=k-p$ average pairs of zeros. The update at $i=k-p+1$ averages the pair $(a_{k-p},a_{k-p+1})$, making both entries negative, and every subsequent update averages two negative entries. Consequently, just before the Type-III update,
\begin{equation*}
a_j^{(p,k-1)}=0\quad(j<k-p),
\qquad
a_j^{(p,k-1)}<0\quad(k-p\leq j\leq k-1).
\end{equation*}
Equation~\eqref{eq:coefficient-zero-sum} and the preceding sign pattern give
\begin{equation*}
a_{k-1}^{(p,k-1)}+a_k^{(p,k-1)}
=-\sum_{j=1}^{k-2}a_j^{(p,k-1)}
>0.
\end{equation*}
This is the numerator of the Type-III scale, whose denominator is positive. Thus the Type-III scale is positive as well.

For $j\leq k-2$, the Type-III map only multiplies $a_j^{(p,k-1)}$ by its positive scale, so the support and signs above persist. It remains to check $j=k-1$. Equation~\eqref{eq:penultimate-coordinate} gives $a_{k-1}^{(p)}=-\Err_p$, where $\Err_p>0$ because the zero-sum condition rules out zero error. Finally, error monotonicity and $\Err_1=1/2$ give
\begin{equation*}
a_k^{(p)}=1-\Err_p\geq\frac12.
\end{equation*}
This completes the induction and proves \eqref{eq:coefficient-signs}.
\end{proof}

\subsection{Proof of Lemma~\ref{lem:cumulative}}
\cumulativeBoundRestated*

\begin{proof}
Recall that using $\bbE[Z_i Z_j]=\mathbf 1\{i=j\}$, we have
\begin{equation*}
\Err_p=\sum_{j<k}(a_j^{(p)})^2+(a_k^{(p)}-1)^2.
\end{equation*}
The last $N+1$ residual coefficients sum to
\begin{equation*}
\sum_{j=k-N}^{k-1}(-a_j^{(p)})+(1-a_k^{(p)})
=1-c_{k-N}^{(p)}.
\end{equation*}
Cauchy--Schwarz on these $N+1$ terms gives
\begin{equation*}
(1-c_{k-N}^{(p)})^2
\leq(N+1)\left(
\sum_{j=k-N}^{k-1}(a_j^{(p)})^2+(1-a_k^{(p)})^2
\right)
\leq(N+1)\Err_p.
\end{equation*}
\end{proof}

\subsection{Proof of Lemma~\ref{lem:pass-recurrence}}
\onePassRecurrenceRestated*

\begin{proof}
Let $\ba^{(p+1,i)}$ be the coefficient vector after the $i$-th learner in pass $p+1$, for $0\leq i\leq k-1$. Thus $\ba^{(p+1,0)}=\ba^{(p)}$, and $\ba^{(p+1,k-1)}$ is the vector returned by learner $pk+k-1$, the last Type-II learner. For $1\leq i\leq k-1$, Lemma~\ref{lem:vertex-maps} gives
\begin{equation*}
\ba^{(p+1,i)}
=T_i\bigl(\ba^{(p+1,i-1)}\bigr)
=s_i\bigl(\ba^{(p+1,i-1)}\bigr)
A_i\ba^{(p+1,i-1)}.
\end{equation*}
As shown in the proof of \eqref{eq:coefficient-signs}, $a_1^{(p)}=0$ and the Type-I scale is $1$. Hence $\ba^{(p+1,1)}=\ba^{(p)}$. Then, by Note~\ref{note:scale-invariance}, we may ignore the positive scales introduced by the Type-II vertices when determining the vector entering the Type-III vertex. More precisely, define
\begin{equation*}
\widetilde{\ba}^{(p+1,k-1)}
:=A_{k-1}\cdots A_2\ba^{(p)}.
\end{equation*}
Then $\ba^{(p+1,k-1)}$ is a positive scalar multiple of $\widetilde{\ba}^{(p+1,k-1)}$. Since $T_k$ is scale invariant, the final update may be computed from $\widetilde{\ba}^{(p+1,k-1)}$. The operators $A_2,\ldots,A_{k-1}$ successively replace the pairs $(a_1,a_2),\ldots,(a_{k-2},a_{k-1})$ by their averages.

Define
\begin{equation*}
m_0=0,\qquad
m_j=\frac{1}{2} \left( m_{j-1}+a_{j+1}^{(p)}\right),
\qquad 1\leq j\leq k-2.
\end{equation*}
At learner $pk+2$, assigned $X_{pk+2}=Z_2-Z_1$, the first two entries are replaced by $m_1=a_2^{(p)}/2$. At the next learner, the second and third entries are replaced by $m_2=(m_1+a_3^{(p)})/2$, and so on. It follows inductively that, after all Type-II learners, the vector entering the final learner (ignoring the scale) is
\begin{equation*}
\widetilde{\ba}^{(p+1,k-1)}
=(m_1,\ldots,m_{k-2},m_{k-2},a_k^{(p)}).
\end{equation*}
The running-average recurrence has the explicit form
\begin{equation}\label{eq:running-average-explicit}
m_j=\sum_{r=2}^{j+1}2^{-(j-r+2)}a_r^{(p)}.
\end{equation}

The Type-III map applies $A_k$, multiplies by a positive scale, which we denote by $\eta_p$, and then adds $\bb=(0,\ldots,0,-\tfrac12,\tfrac12)$. Since $\bw$ omits coordinate $k$, the additive term $\bb$ affects only $w_1$ among the profile coordinates. By Lemma~\ref{lem:invariants}, specifically \eqref{eq:penultimate-coordinate},
\begin{equation*}
w_1^{(p+1)}=-a_{k-1}^{(p+1)}=\Err_{p+1}.
\end{equation*}

For $2\leq i\leq k-1$, set $j=k-i$, substitute $r=k-(i-1+t)$, and use \eqref{eq:edge-profile} with $a_1^{(p)}=0$. Equation~\eqref{eq:running-average-explicit} then gives
\begin{equation*}
w_i^{(p+1)}
= -\eta_p m_{k-i}
= \eta_p \sum_{t=0}^{k-i}2^{-(t+1)}w_{i-1+t}^{(p)}
= \eta_p (\cT\bw^{(p)})_i.
\end{equation*}

We have thus proved \eqref{eq:vector-recurrence}, which, on iteration, gives \eqref{eq:positive-mixture}.
\end{proof}

\subsection{Proof of Theorem~\ref{thm:regression-improved-constants}}
\regressionConstantsRestated*

\begin{proof}
Retaining the exact dependence on $N$ in \eqref{eq:c-tail-bound} (which was derived for $p \geq 4$, but holds for $p = 2, 3$ as well), Lemma~\ref{lem:cumulative} gives, for every integer $1\leq N\leq k-1$,
\begin{equation}\label{eq:constant-family-N}
\Err_p
\geq
\frac{\left(1-\exp\{-N(N-1)/(4(p-1))\}\right)^2}{N+1}.
\end{equation}
Take
\begin{equation*}
N=\left\lceil t\sqrt{p-1}+1\right\rceil
\end{equation*}
in \eqref{eq:constant-family-N}. Then
\begin{equation*}
N(N-1)\geq t^2(p-1),
\qquad
N+1\leq t\sqrt{p-1}+3,
\end{equation*}
which proves \eqref{eq:constant-family-t}.
\end{proof}

\subsection{Proof of Lemma~\ref{lem:spike-tail}}\label{app:partial-sums-proof}
\partialSumsRestated*

\begin{proof}[Proof of Lemma~\ref{lem:spike-tail}]

We will obtain a generating function $V(x, y)$ containing the coefficients $(\cT^p \be_1)_i$, and then obtain the cumulative sum $\frakc(p, N)$ from $V(x, y)$. For the rest of this proof, $x, y$ are free variables, unrelated to their usage elsewhere in the manuscript. 

\paragraph{Functional equation for $V$.}
For $1 \leq i, j \leq k - 1$, define \[v_{i,j} =\left(\cT^{i-1}\be_1\right)_j,\] and note the boundary conditions $v_{1,1}=1$, $v_{1,j}=0$ for $j \geq 2$, and $v_{i,1}=0$ for $i\geq2$. For $i, j \geq k$, we define $v_{i, j} = 0$. Now, for $i,j\geq2$, the definition of $\cT$ gives
\begin{equation*}
v_{i,j}
=\sum_{t\geq0}2^{-(t+1)}v_{i-1,j-1+t}, 
\end{equation*}
which can be rewritten as the recurrence 
\begin{equation}\label{eq:alternate-local-recurrence}
v_{i,j}
=\frac12v_{i-1,j-1}+\frac12v_{i,j+1}.
\end{equation}
Now, we arrange the coefficients $v_{i, j}$ on a formal power series by defining
\begin{equation*}
V(x,y)=\sum_{i,j\geq1}v_{i,j}x^iy^j.
\end{equation*}
Further, let $V_2(x)=\sum_{i\geq1}v_{i,2}x^i$ be the formal power series obtained from the second column of the coefficient matrix. Now using the recurrence \eqref{eq:alternate-local-recurrence}, we get
\begin{equation*}
\sum_{i,j\geq2}v_{i,j}x^iy^j=\frac12 \sum_{i,j\geq2} v_{i-1,j-1} x^iy^j + \frac12 \sum_{i,j\geq2} v_{i,j+1} x^iy^j, 
\end{equation*}
which obtains
\begin{equation*}
V(x,y)-xy
=\frac{xy}{2}V(x,y)
+\frac{1}{2y}\left(V(x,y)-xy-y^2V_2(x)\right)
\end{equation*}
using the boundary conditions.
Equivalently,
\begin{equation}\label{eq:alternate-kernel-equation}
\left(2y-xy^2-1\right)V(x,y)
=2xy^2-xy-y^2V_2(x).
\end{equation}

\paragraph{Solving for $V$.} We will now apply the so-called \emph{kernel method} to solve \eqref{eq:alternate-kernel-equation} for $V(x, y)$. Let
\begin{equation*}
P(x,y)=2y-xy^2-1,
\qquad
Q(x,y)=2xy^2-xy-y^2V_2(x),
\end{equation*}
so that \eqref{eq:alternate-kernel-equation} is $P(x,y)V(x,y)=Q(x,y)$. Letting
\begin{equation*}
U(x)=\frac{1-\sqrt{1-x}}{x}
\qquad\text{and}\qquad
\widetilde U(x)=\frac{1+\sqrt{1-x}}{x},
\end{equation*}
be the roots of $P(x, y)$ viewed as a polynomial in $y$, we have
\begin{equation*}
P(x,y)=-x(y-U(x))(y-\widetilde U(x))
\qquad\text{and}\qquad
Q(x,y)=y\left((2x-V_2(x))y-x\right).
\end{equation*}

Note that $U(x)$ is the only formal power series among the roots, the other root $\widetilde U(x)$ contains the term $1/x$. It is tempting to substitute $y = U(x)$ in \eqref{eq:alternate-kernel-equation} to obtain $V_2(x)$, and eventually $V(x, y)$. However, we must be cautious, as $[x^0] U(x) \neq 0$, and the composition of power series might not be well defined. Fortunately, we observe from the recurrence \eqref{eq:alternate-local-recurrence} that $v_{i,j}=0$ for $j \geq i + 1$. Hence, the coefficient of $x^a$ in $V(x,y)$, denoted $[x^a] V(x, y)$, is a polynomial in $y$ for all $a \geq 0$, and substituting a formal power series for $y$ yields a valid formal power series $V(x, U(x))$. 

Since $P(x,U(x))=0$, \eqref{eq:alternate-kernel-equation} shows that $Q(x,U(x))=0$. We thus have $(2x-V_2(x))U(x)=x$, which in turn shows
\begin{equation*}
Q(x,y)=\frac{x}{U(x)}y(y-U(x)).
\end{equation*}
Consequently, recalling $V(x, y) = Q(x, y) / P(x, y)$, we have
\begin{align*}
V(x,y) =\frac{(x/U(x))y(y-U(x))}
{-x(y-U(x))(y-\widetilde U(x))} =\frac{y}{U(x)\widetilde U(x)-yU(x)} =\frac{xy}{1-xyU(x)} =\frac{xy}{1-y(1-\sqrt{1-x})},
\end{align*}
where the last equality uses $U(x)\widetilde U(x)=1/x$.

\paragraph{Extracting the partial sum.} To obtain $\frakc(p,N)$, we see that \begin{equation*}
[x^{p+1}y^N](1 - y)^{-1} V(x,y)
=\sum_{j=1}^{N}v_{p+1,j}
=\frakc(p,0)-\frakc(p,N), 
\end{equation*}
from which the required partial sum is obtained as
\begin{align*}
\frakc(p,N)
&=[x^{p+1}y^N] (1 - y)^{-1}(V(x,1)-V(x,y)).
\end{align*}

Letting $u(x) = 1 - \sqrt{1 - x}$, we have 
\begin{equation*}
\frac{V(x,1)-V(x,y)}{1-y}
=\frac{x}{(1-u(x))(1-y u(x))} = \frac{x}{1 - u(x)} \sum_{r \geq 0} y^r u(x)^r, 
\end{equation*}
which implies
\begin{equation}
\frakc(p,N)
=[x^p]\frac{u(x)^N}{1-u(x)}. \label{eq:lif_apply}
\end{equation}

We would ideally like to apply the Lagrange inversion formula to solve \eqref{eq:lif_apply} using the implicit form 
\begin{equation}
u(x) = x \cdot (2 - u(x))^{-1}. \label{eq:uimplicit}
\end{equation}
However, direct application leads to a cumbersome calculation involving sums of binomial coefficients. It is easier to apply a variant instead (see \citep[Appendix A.6, A.13]{flajolet-sedgewick}). To do this, we differentiate \eqref{eq:uimplicit} to get $u'(x) = \frac{1}{2(1 - u(x))}$, which then gives
\begin{equation*}
\frac{d}{dx}u(x)^{N+1}
=(N+1)u(x)^Nu'(x)
=\frac{N+1}{2}\frac{u(x)^N}{1-u(x)}.
\end{equation*}
Let $H$ be any formal power series $H(x)$. Then we have $[x^p]H'(x)=(p+1)[x^{p+1}]H(x)$. Therefore, we have
\begin{align*}
[x^p]\frac{u(x)^N}{1-u(x)} =\frac{2}{N+1}[x^p]\frac{d}{dx}u(x)^{N+1} =\frac{2(p+1)}{N+1}[x^{p+1}]u(x)^{N+1}.
\end{align*}
Now recall the Lagrange inversion formula which gives $[x^a] H(u) = \frac{1}{a} [u^{a-1}](H'(u) \phi(u)^a)$ for a formal power series $u(x)$ satisfying the functional form $u = x \phi(u)$, $[x^0] u = 0$, any formal power series $H$, and $\phi(0) \neq 0$. Applying this with $\phi(u) = (2 - u)^{-1}$, and $H(u) = u^{N + 1}$, we have 
\begin{equation*}
[x^{p+1}]u(x)^{N+1}
=\frac{N+1}{p+1}[u^{p-N}](2-u)^{-(p+1)}.
\end{equation*}
Hence, for $0\leq N\leq p$, we use the generalized binomial coefficients to get
\begin{align*}
[x^p]\frac{u(x)^N}{1-u(x)} = 2[u^{p-N}](2-u)^{-(p+1)} = 2\binom{-(p+1)}{p-N}(-1)^{p-N}2^{-(p+1)-(p-N)} = 2^{N-2p}\binom{2p-N}{p}.
\end{align*}
Further, $[x^0]u(x) = 0$, implies that the coefficient $[x^p]\frac{u(x)^N}{1 - u(x)} = 0$ for all $p < N$.  Finally,
\begin{equation*}
\frac{\frakc(p,N)}{\frakc(p,0)}
=2^N\frac{\binom{2p-N}{p}}{\binom{2p}{p}}
=\prod_{j=0}^{N-1}
\left(1-\frac{j}{2p-j}\right) 
\leq
\exp \left(-\sum_{j=0}^{N-1}\frac{j}{2p}\right)
=\exp \left(-\frac{N(N-1)}{4p}\right),
\end{equation*}
for all $1 \leq N \leq p$, where we have used $1-x\leq e^{-x}$ and $2p-j\leq2p$ for the inequality. 
\end{proof}

\section{Omitted Proofs for Section~\ref{sec:general-convex}}

\subsection{Proof of Lemma~\ref{lem:gaussian-lower-bound}}

\begin{lemma}\label{lem:gaussian-lower-bound}
Let $G\sim\cN(0,1)$. For every $0\leq r\leq2$,
\begin{equation*}
\min\left\{
\bbE[\sigma'(rG)],
\bbE[G^2\sigma'(rG)]
\right\}
\geq
\frac{1}{12\sqrt3}
>0.048.
\end{equation*}
\end{lemma}

\begin{proof}
Let $S\sim\operatorname{Unif}\{-1,1\}$ be a Rademacher random variable. Since $\bbE[S]=0$ and $S\in[-1,1]$, Hoeffding's lemma gives, for every $u\in\mathbb R$,
\begin{equation}
\frac{e^u+e^{-u}}{2}
=\bbE[e^{uS}]
\leq\exp\left(\frac{u^2(1-(-1))^2}{8}\right) \label{eq:coshbd}
=e^{u^2/2}.
\end{equation}
From the definition $\sigma(t)=(1+e^{-t})^{-1}$, applying \eqref{eq:coshbd} with $u=t/2$ gives
\begin{equation*}
\sigma'(t)
=\frac{e^{-t}}{(1+e^{-t})^2}
=\frac{1}{(e^{t/2}+e^{-t/2})^2}
=\frac14\left(\frac{e^{t/2}+e^{-t/2}}{2}\right)^{-2}
\geq\frac14\left(e^{t^2/8}\right)^{-2}
=\frac14e^{-t^2/4}.
\end{equation*}
Consequently, for every $r\geq0$,
\begin{equation}\label{eq:logistic-pointwise-gaussian-bound}
\sigma'(rG)\geq\frac14e^{-r^2G^2/4}.
\end{equation}
For $a\geq0$, first write each expectation against the standard Gaussian density:
\begin{align*}
\bbE[e^{-aG^2}] &= \frac{1}{\sqrt{2\pi}}\int_{\mathbb R}e^{-(1+2a)x^2/2}\,dx,\\
\bbE[G^2e^{-aG^2}] &=\frac{1}{\sqrt{2\pi}}\int_{\mathbb R}x^2e^{-(1+2a)x^2/2}\,dx.
\end{align*}
Under the change of variables $z=\sqrt{1+2a}\,x$, we have $dx=(1+2a)^{-1/2}\,dz$ and $x^2=(1+2a)^{-1}z^2$. Hence
\begin{align*}
\bbE[e^{-aG^2}]
&=(1+2a)^{-1/2}\frac{1}{\sqrt{2\pi}}\int_{\mathbb R}e^{-z^2/2}\,dz =(1+2a)^{-1/2},\\
\bbE[G^2e^{-aG^2}]
&=(1+2a)^{-3/2}\frac{1}{\sqrt{2\pi}}\int_{\mathbb R}z^2e^{-z^2/2}\,dz =(1+2a)^{-3/2},
\end{align*}
where the last equalities use that the standard Gaussian density integrates to one and $\bbE[G^2]=1$, respectively.
Taking $a=r^2/4$ in these identities and using \eqref{eq:logistic-pointwise-gaussian-bound} gives
\begin{align*}
\bbE[\sigma'(rG)]
&\geq\frac14\left(1+\frac{r^2}{2}\right)^{-1/2},\\
\bbE[G^2\sigma'(rG)]
&\geq\frac14\left(1+\frac{r^2}{2}\right)^{-3/2}.
\end{align*}
Since $1+r^2/2\geq1$ and $r\leq2$, both expectations are at least
\begin{equation*}
\frac14\left(1+\frac{r^2}{2}\right)^{-3/2}
\geq
\frac14(3)^{-3/2}
=\frac{1}{12\sqrt3}.
\end{equation*}
\end{proof}

\subsection{Proof of Lemma~\ref{lem:F-gradient}}

\begin{lemma}\label{lem:F-gradient}
Fix $\gamma_0\in[1/2,1]$. Let $Z\sim\cN(0,1)$ and $E\sim\cN(0,\gamma_0(1-\gamma_0))$ be independent, and define
\[
F(\gamma)
=\bbE\ell_{\rm logistic}\bigl(\gamma(E+\gamma_0Z),Z\bigr).
\]
Then $F'(0)<0$ and $F'(1)\geq0$.
\end{lemma}

\begin{proof}
Let $X=E+\gamma_0Z$. Since $E$ and $Z$ are independent, mean zero gaussian random variables, we have that $X$ is a mean zero gaussian with variance
\[
{\rm Var}(X)
={\rm Var}(E)+\gamma_0^2{\rm Var}(Z)
=\gamma_0(1-\gamma_0)+\gamma_0^2
=\gamma_0.
\]
Further, ${\rm Cov}(X,Z)=\gamma_0$, and 
\[
F'(\gamma)
=\bbE\left[
\bigl(\sigma(\gamma X)-\sigma(Z)\bigr)X
\right]
=\gamma_0\left\{
\gamma\bbE[\sigma'(\gamma X)]
-\bbE[\sigma'(Z)]
\right\},
\]
where the last equality follows by applying Stein's identity in the form $\bbE [f(E_1) E_2] = {\rm Cov}(E_1, E_2) \bbE f'(E_1)$ for centered jointly gaussian random variables $(E_1, E_2)$. Let $G \sim \cN(0, 1)$. At $\gamma=0$,
\[
F'(0)
=-\gamma_0\bbE[\sigma'(G)]
<0,
\]
because $\gamma_0>0$ and $\sigma'(G)>0$ almost surely. At $\gamma=1$, we get
\[
F'(1)
=\gamma_0\left\{
\bbE[\sigma'(\sqrt{\gamma_0}G)]
-\bbE[\sigma'(G)]
\right\}.
\]
Fix any $g\in\bbR$. Since $\gamma_0\leq1$, we have $|\sqrt{\gamma_0}g|\leq|g|$. Moreover, $\sigma'(t) = (e^{t/2}+e^{-t/2})^{-2}$, i.e., $\sigma'(t) = \sigma'(|t|)$ and $\sigma'$ decreases on $[0,\infty)$. Hence
\[
\sigma'(\sqrt{\gamma_0}g)
=\sigma'(|\sqrt{\gamma_0}g|)
\geq\sigma'(|g|)
=\sigma'(g).
\]
Taking expectations and using $\gamma_0>0$ shows that $F'(1)\geq0$.
\end{proof}

\bibliographystyle{plainnat}
\bibliography{refs}

\begin{thebibliography}{5}
\providecommand{\natexlab}[1]{#1}
\providecommand{\url}[1]{\texttt{#1}}
\expandafter\ifx\csname urlstyle\endcsname\relax
  \providecommand{\doi}[1]{doi: #1}\else
  \providecommand{\doi}{doi: \begingroup \urlstyle{rm}\Url}\fi

\bibitem[Bateni et~al.(2026)Bateni, Hadizadeh, Hajiaghayi, JafariRaviz, and
  Taherijam]{bateninetworked}
Mohammadhossein Bateni, Zahra Hadizadeh, MohammadTaghi Hajiaghayi, Mahdi
  JafariRaviz, and Shayan Taherijam.
\newblock Networked information aggregation for binary classification.
\newblock In \emph{Forty-third International Conference on Machine Learning},
  2026.

\bibitem[Flajolet and Sedgewick(2009)]{flajolet-sedgewick}
P.~Flajolet and R.~Sedgewick.
\newblock \emph{Analytic Combinatorics}.
\newblock Cambridge University Press, 2009.
\newblock URL \url{https://doi.org/10.1017/CBO9780511801655}.

\bibitem[Huber(1964)]{huber1964}
P.~J. Huber.
\newblock Robust estimation of a location parameter.
\newblock \emph{The Annals of Mathematical Statistics}, 35\penalty0
  (1):\penalty0 73--101, 1964.
\newblock URL \url{https://doi.org/10.1214/aoms/1177703732}.

\bibitem[Kearns et~al.(2026)Kearns, Roth, and Ryu]{krr}
M.~Kearns, A.~Roth, and E.~Ryu.
\newblock Networked information aggregation via machine learning.
\newblock In \emph{Proceedings of the Annual {ACM--SIAM} Symposium on Discrete
  Algorithms ({SODA})}, 2026.

\bibitem[Li and Duan(1989)]{li-duan1989}
K.-C. Li and N.~Duan.
\newblock Regression analysis under link violation.
\newblock \emph{The Annals of Statistics}, 17\penalty0 (3):\penalty0
  1009--1052, 1989.
\newblock URL \url{https://doi.org/10.1214/aos/1176347254}.

\end{thebibliography}

\end{document}